\documentclass{article}
\usepackage{amsfonts,latexsym,amssymb,amsmath,amsthm,epsfig,amscd,bbm}
\usepackage{enumitem}
\usepackage{hyperref}
\usepackage[margin=0.9in]{geometry}

\usepackage[all,cmtip]{xy}
\newcommand{\dc}[1]{*++[o][F.]{#1}}

\newcommand{\zset}{\mathbb Z}

\newcommand{\rset}{\mathbb R}

\newcommand{\hset}{\mathbb H}
\newcommand{\ep}{{\varepsilon}}

\theoremstyle{plain}
\newtheorem{theorem}{Theorem}
\newtheorem{lemma}[theorem]{Lemma}

\theoremstyle{definition}

\newtheorem{defn}[theorem]{Definition}

\newcommand{\One}{\mathbbm{1}}

\newcommand{\suppress}[1]{}
\newcommand{\mm}{{\mathbf{m}}}

\newcommand{\oep}{\overline{\ep}}
\newcommand{\nae}{\text{NAE}}
\DeclareMathOperator{\diag}{diag}
\DeclareMathOperator{\Span}{span}
\DeclareMathOperator{\rowspace}{Rowspace}

\makeatletter
\def\Bigbar#1{\mathrel{\left|\vphantom{#1}\right.\n@space}}

\makeatother

\DeclareMathOperator{\rank}{rank}

\newcommand{\be}{\begin{equation}}
\newcommand{\ee}{\end{equation}}
\newcommand{\bea}{\begin{eqnarray}}
\newcommand{\eea}{\end{eqnarray}}
\newcommand{\bean}{\begin{eqnarray*}}
\newcommand{\eean}{\end{eqnarray*}}

\begin{document}
\title{Hadamard Extensions and the Identification of Mixtures of Product Distributions}
\author{
Spencer L. Gordon\thanks{Engineering and Applied Science, California Institute of Technology, {\tt slgordon@caltech.edu}.} \and
Leonard J. Schulman\thanks{Engineering and Applied Science, California Institute of Technology, {\tt schulman@caltech.edu}. Research supported in part by NSF grant CCF-1909972.}
}
 \maketitle

\begin{abstract} The Hadamard Extension of a matrix is the matrix consisting of all Hadamard products of subsets of its rows. This construction arises in the context of identifying a mixture of product distributions on binary random variables: full column rank of such extensions is a necessary ingredient of identification algorithms. 
We provide several results concerning when a Hadamard Extension has full column rank. 
\end{abstract}

\section{Introduction}\label{sec: intro}
The Hadamard product for row vectors $u=(u_1,\ldots,u_k)$, $v=(v_1,\ldots,v_k)$ is the mapping $\odot: \rset^k \times \rset^k  \to \rset^k$ given by
\begin{align*} u \odot v & := (u_1v_1,\ldots,u_kv_k) \end{align*}
The identity for this product is the all-ones vector $\One$. We associate with vector $v$ the linear
operator $v_\odot = \diag(v)$, a $k \times k$ diagonal matrix, so that 
\[ u \cdot v_\odot =v \odot u .\]

Throughout this paper $\mm$ is a real matrix with row set $[n]:=\{1,\ldots,n\}$ and column set $[k]$; write $\mm_{i}$ for a row and $\mm^{j}$ for a column.

As a matter of notation, for a matrix $Q$ and nonempty sets $R$ of rows and $C$ of columns, let $Q|_R^C$ be the restriction of $Q$ to those columns and rows (with either index omitted if all rows or columns are retained).

\begin{defn} The Hadamard Extension of $\mm$, written 
$\hset(\mm)$, is the $2^n \times k$ matrix 
with rows $\mm_S$ for all $S\subseteq [n]$, where, for $S=\{i_1,\ldots,i_\ell\}$,
$\mm_{S}= \mm_{i_1} \odot \cdots \odot \mm_{i_\ell}$; equivalently $\mm_{S}^j=\prod_{i\in S} \mm_{i}^j$. (In particular $\mm_\emptyset=\One$.)
\end{defn}
 
This construction has arisen recently in learning theory~\cite{CM19,GMRS20b} where it is essential to source identification for a mixture of product distributions on binary random variables. We explain the connection further in Section~\ref{sec:motivation}. Motivated by this application, we are interested in the following two questions:

(1) If $\hset(\mm)$ has full column rank, must there exist a subset $R$ of the rows, of bounded size, such that $\hset(\mm|_R)$ has full column rank?

(2) In each row of $\mm$, assign distinct colors to the distinct real values. Is there a condition on the coloring that ensures $\hset(\mm)$ has full column rank?

In answer to the first question we show in Section~\ref{sec:had}:

\begin{theorem} If $\hset(\mm)$ has full column rank then there is a set $R$ of no more than $k-1$ of the rows of $\mm$, such that $\hset(\mm|_R)$ has full column rank. \label{knockdown} \end{theorem}

Considering the more combinatorial second question, observe that 
 if $\mm$ possesses two identical columns then the same is true of 
$\hset(\mm)$, and so it cannot be full rank. Extending this further, suppose there are three columns $C$ in which only one row $r$ has more than one color. Then 
$\rowspace \hset(\mm|^C)$ is spanned by $\One|^C$ and $r|^C$, so again $\hset(\mm)$ cannot be full rank.
Motivated by these necessary conditions, set:
\begin{defn}
For a matrix $Q$ let $\nae(Q)$ be the set of nonconstant rows of $Q$ (NAE=``not all equal''); let
 $\ep(Q|^C)=|\nae(Q|^C)|-|C|$; and let $\oep(Q)=\min_{C \neq \emptyset} \ep(Q|^C)$. 
If $\oep(Q)\geq -1$ we say $Q$ satisfies the NAE condition. 
\end{defn} 
In answer to the second question we have the following:
\begin{theorem} \label{naethm}
If $\mm$ satisfies the NAE condition then
\begin{enumerate}[label=(\alph*)] \item There is a restriction of $\mm$ to some $k-1$ rows $R$ such that $\oep(\mm|_R)=-1$. \label{parta}
 \item $\hset(\mm)$ is full column rank. 
\label{partb} \end{enumerate}
(As a consequence also $\hset(\mm|_R)$ is full column rank.)
\end{theorem} Apparently the only well-known example of the NAE condition is when $\mm$ contains $k-1$ rows which are identical and whose entries are all distinct. Then the vectors $\mm_\emptyset, \mm_{\{1\}}, \mm_{\{1,2\}},\ldots, \mm_{\{1,\ldots,k-1\}}$ form a nonsingular Vandermonde matrix. This example shows that the bound of $k-1$ in \ref{parta} is best possible.

For another example in which the NAE condition ensures that $\rank \hset(\mm)=k$, take
 the $(k-1)$-row matrix with $\mm_{i}^j=1$ for $i\leq j$ and $\mm_{i}^j=1/2$ for $i>j$.
 Here the NAE condition is only minimally satisfied, in that for every $\ell\leq k$ there are $\ell$ columns $C$ s.t.\ $\ep(\mm|^C)=-1$. 

For $k>3$ the NAE condition is no longer necessary for $\hset(\mm)$ to have full column rank.
E.g., for $k=2^\ell$, the $\ell \times k$ ``Hamming matrix'' $\mm_{i}^j=(-1)^{j_i}$ where $j$ is an $\ell$-bit string $j=(j_1,\ldots,j_\ell)$, forms $\hset(\mm)=$ the Fourier transform for the group $(\zset/2)^\ell$ (often called a Hadamard matrix), which is invertible. Furthermore, almost all (in the sense of Lebesgue measure) $\lceil \lg k \rceil \times k$ matrices $\mm$ form a full-rank $\hset(\mm)$. (This is because $\det \hset(\mm)$ is a polynomial in the entries of $\mm$, and the previous example shows the polynomial is nonzero.)
Despite this observation, the Vandermonde case, in which $k-1$ rows are required, is very typical, as it is what arises in $\hset(\mm)$ for a mixture model of observables $X_i$ that are iid conditional on a hidden variable.

\section{Some Theory for Hadamard Products, and a Proof of Theorem~\ref{knockdown}} \label{sec:had}

For $v \in \rset^k$ and $U$ a subspace, extend the definition $v_\odot$ to
\[ v_\odot (U)= \{u \cdot v_\odot: u\in U\}\]
and introduce the notation
\[ v_{\bar\odot}(U)=\Span\{U \cup v_\odot(U)\}. \]

We want to understand which subspaces $U$
are invariant under $v_{\bar\odot}$. 
Let $v$ have distinct values $\lambda_1>\ldots>\lambda_\ell$ for $\ell\leq k$. 
Let the polynomials $p_{v,i}$ ($i=1,\ldots,\ell$) of degree $\ell-1$ be the Lagrange interpolation polynomials for these values, so $p_{v,i}(\lambda_j)=\delta_{ij}$ (Kronecker delta). Let $B(v)$ denote the partition of $[k]$ into 
blocks $B(v)_{(i)}=\{j: v_j=\lambda_i\}$.
Let $V_{(i)}$ be the space spanned by the elementary basis vectors in $B(v)_{(i)}$, and $P_{(i)}$ the projection onto $V_{(i)}$ w.r.t.\ standard inner product.
We have the matrix equation
 \[ p_{v,i}(v_\odot) =P_{(i)}. \] 
 The collection of all linear combinations of the matrices $P_{(i)}$ is a commutative algebra, the \emph{$B(v)$ projection algebra}, which we denote $A_{B(v)}$. The identity of the algebra is $I=\sum P_{(i)}$.

\begin{defn} A subspace of $\rset^k$ \textit{respects} $B(v)$ if it is spanned by vectors each of which lies in some $V_{(i)}$.
\end{defn}
For $U$ respecting $B(v)$ write $U=\Span(\bigcup U_{(i)})$ for $U_{(i)} \subseteq V_{(i)}$.  
Let $D_{(i)}=(U_{(i)})^\perp \cap V_{(i)}$. 
Then $(U_{(i)})^\perp =D_{(i)} \oplus \bigoplus_{j \neq i} V_{(j)}$. 

\begin{lemma} Subspace $U^\perp$ respects $B(v)$ if $U$ does.\end{lemma}
\begin{proof} In general, $(\Span(W \cup W'))^\perp=W^\perp \cap W'^\perp$. 
So $U^\perp= \bigcap (U_{(i)})^\perp = \bigoplus D_{(i)}$. 
\end{proof}

\begin{lemma} Subspace $U$ respects $B(v)$ iff $U= \bigoplus (P_{(i)} U)$. \label{projs} \end{lemma} \begin{proof} ($\Leftarrow$): Because this gives an explicit representation of $U$ as a direct sum of subspaces each restricted to some $V_{(i)}$. ($\Rightarrow$): By definition $U$ is spanned by some collection of subspaces $V'_{(i)}\subseteq V_{(i)}$; since these subspaces are necessarily orthogonal, $U=\bigoplus V'_{(i)}$. Moreover, since $P_{(i)}$ annihilates $V_{(j)}, j\neq i$, and is the identity on $V_{(i)}$, it follows that each $V'_{(i)}= P_{(i)} U$.
\end{proof}

\begin{theorem} Subspace $U$ is invariant under $v_{\bar \odot}$ iff $U$ respects $B(v)$.
\label{invU} \end{theorem}

\begin{proof} 

($\Leftarrow$): 
It suffices to show $U^\perp$ is invariant under $v_{\bar \odot}$. By the previous lemma, it is equivalent to suppose that $U^\perp$ respects $B(v)$. So let $d\in U^\perp$ and write $d=\sum d_{i}, d_i \in D_{(i)}$. Then $v \odot d_i =\lambda_i d_i \in D_{(i)}$. So $v \odot d = \sum v \odot d_i \in \bigoplus D_{(i)} = U^\perp$. 

($\Rightarrow$): If $U=v_{ \bar \odot}(U)$ then these also equal
$v_{ \bar \odot}(v_{ \bar \odot}(U))$, etc., so $U$ is an invariant space of $A_{B(v)}$, meaning, $aU \subseteq U$ for any $a \in A_{B(v)}$. In particular for $a=P_{(i)}$. So $U \supseteq \bigoplus (P_{(i)}U)$. On the other hand, since $\sum P_{(i)} = I$, $U = (\sum P_{(i)})U \subseteq \bigoplus (P_{(i)}U)$. So $U =\bigoplus (P_{(i)}U)$. Now apply Lemma~\ref{projs}.
 \end{proof}

The symbol $\subset$ is reserved for strict inclusion.
 
\begin{lemma} If $S, T \subseteq [n]$ and 
$\rowspace \hset(\mm|_S) \subset \rowspace \hset(\mm|_{S \cup T})$,
then there is a row $t \in T$ such that
$\rowspace \hset(\mm|_S) \subset \rowspace \hset(\mm|_{S \cup \{t\}})$.
\label{lem:St} \end{lemma}
\begin{proof} Without loss of generality $S,T$ are disjoint. Let $T'\subseteq T$ be a smallest set s.t.\ $\exists S'\subseteq S$ s.t.\ $\mm_{S'} \odot \mm_{T'} \notin 
\rowspace \hset(\mm_S)$.
Select any $t\in T'$ and write $\mm_{S'} \odot \mm_{T'} = \mm_{S'} \odot \mm_{T'-\{t\}} \odot \mm_t$. By minimality of $T'$, $\mm_{S'} \odot \mm_{T'-\{t\}} \in 
\rowspace \hset(\mm_S)$.
But then $\mm_{S'} \odot \mm_{T'} \in
\rowspace \hset(\mm_{S\cup \{t\}})$, so
$\rowspace \hset(\mm|_S) \subset \rowspace \hset(\mm|_{S \cup \{t\}})$.
\end{proof}

Theorem~\ref{knockdown} is now a consequence of Lemma~\ref{lem:St}. \qed

It follows from 
Theorem~\ref{knockdown} that we can check whether $\rank \hset(\mm)=k$ in time $O(n)^k$ by computing $\rank \hset(\mm|_S)$ for
each $S \in \binom{[n]}{k-1}$.

\section{Combinatorics of the NAE Condition: Proof of Theorem~\ref{naethm}~\ref{parta}} \label{sec:nae}
Recall we are to show:~\ref{naethm}~\ref{parta}: \textit{If $\oep(\mm)\geq -1$ then $\mm$ has a restriction to some $k-1$ rows on which $\oep=-1$. }
\begin{proof} We induct on $k$. The (vacuous) base-case is $k=1$.
For $k>1$, we induct on $n$, with base-case $n=k-1$.

Supposing the Theorem fails for $k$, $k>1$, let 
 $\mm$ be a $k$-column counterexample with least $n$. Necessarily every row is in $\nae(\mm)$, and
 $n>k-1\geq 1$. We will show $\mm$ has a restriction $\mm'$ to 
$n-1$ rows, for which $\oep(\mm')\geq -1$;
this will imply a contradiction because, by minimality of $\mm$, $\mm'$ has a restriction to $k-1$ rows on which $\oep= -1$.

If $\oep(\mm)\geq 0$ then we can remove any single row of $\mm$ and still satisfy $\oep\geq-1$. 

Otherwise, $\oep(\mm)=-1$, so there is a nonempty $S$ such that $|\nae(\mm|^S)|=|S|-1$; choose a largest such $S$. It cannot be that $S=[k]$ (as then $n=k-1$). 
Arrange the rows $\nae(\mm|^S)$ as the bottom $|S|-1$ rows of the matrix. As discussed earlier, for the NAE condition one may 
regard the distinct real values in each row of $\mm$ simply as distinct colors; relabel the colors in each row above $\nae(\mm|^S)$ so the color above $S$ is called ``white.'' (There need be no consistency among the real numbers called white in different rows.) See Figure~\ref{minimalfig}.

\begin{figure}[ht] \centering
\includegraphics[scale=0.75]{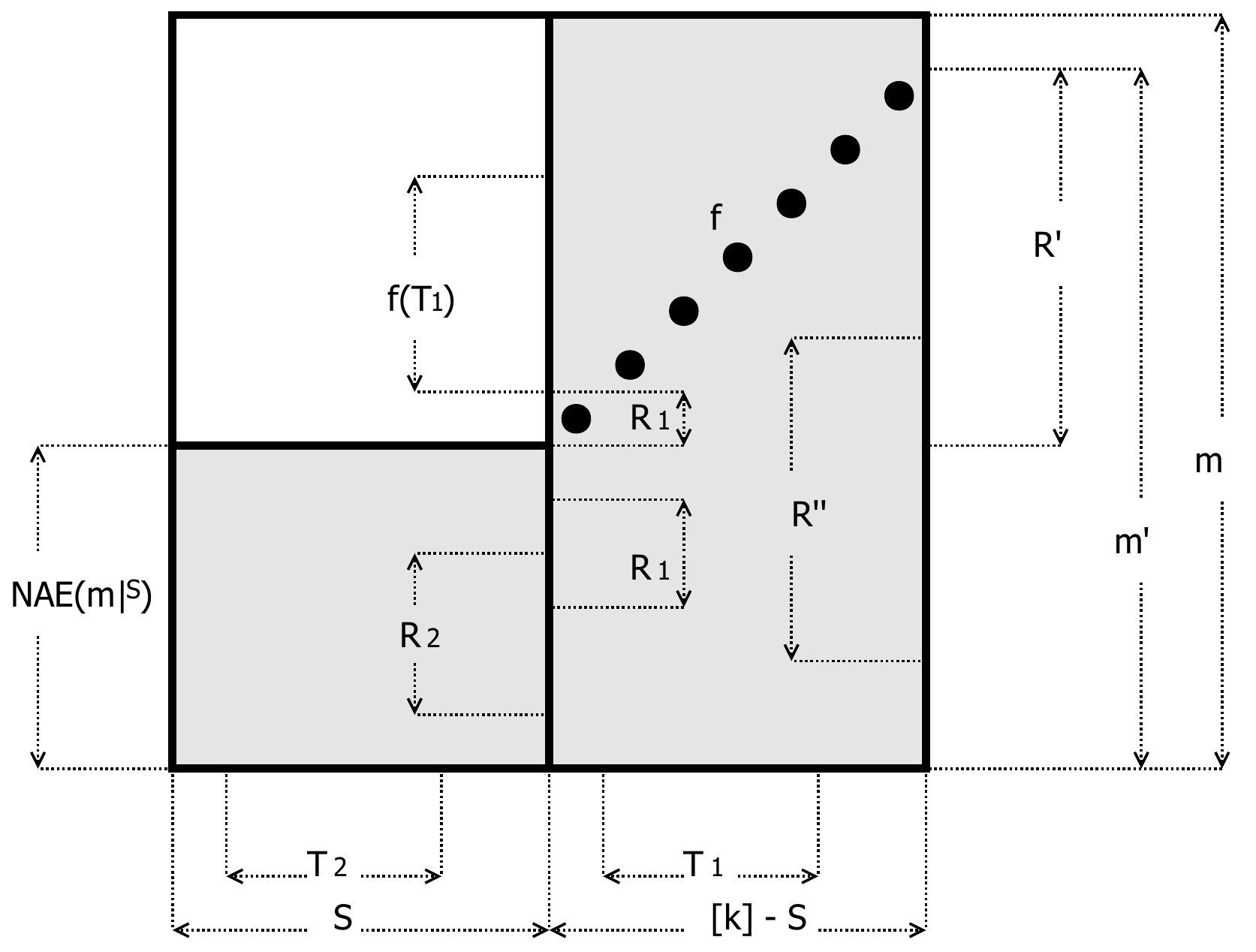}
\caption{Argument for Theorem~\ref{naethm}~\ref{parta}. Upper-left region is white. Entries $(t,f(t))$ are not white.} \label{minimalfig} \end{figure}

Due to the maximality of $|S|$, there is no 
white rectangle 
on $\ell$ columns and $n-|S|-\ell+1$ rows
inside $\mm|_{[n]-\nae(\mm|^S)}^{[k]-S}$ for any $\ell\geq 1$. That is to say, if we form a bipartite graph on right vertices corresponding to the columns $[k]-S$, and left vertices corresponding to the rows $[n]-\nae(\mm|^S)$, with non-white cells being edges, then any subset of the right vertices of size $\ell\geq 1$ has at least $\ell+1$ neighbors within the left vertices. 

By the induction on $k$ 
(since $S\neq\emptyset$), for the set of columns $[k]-S$ there is a set $R''$ of $k-|S|-1$ rows such that $\oep(\mm|_{R''}^{[k]-S})=-1$.
Together with the rows of $\nae(\mm|^S)$ this amounts to at most $k-2$ rows, so since $n\geq k$, we can find two rows outside this union; delete either one of them, leaving a  matrix $\mm'$ with $n-1$ rows. This matrix has the rows $\nae(\mm|^S)$ at the bottom, and $n-|S|$ remaining rows which we call $R'$. The lemma will follow by showing that $\oep(\mm')\geq -1$. 

In $\mm'$, the induced bipartite graph on right vertices $[k]-S$ and left vertices $R'$ has the property that any right subset of size $\ell\geq 1$ has a neighborhood of size at least $\ell$ in $R'$. Applying Hall's Marriage Theorem, there is an injective $f:[k]-S \to R'$ employing only edges of the graph.

Now consider any set of columns $T$, $T=T_1\cup T_2, T_1\subseteq [k]-S, T_2 \subseteq S$. We need to show that $\ep(\mm|^T)\geq -1$. Let $R_1=\nae(\mm|^{T_1})\cap R''$, $R_2=\nae(\mm|^{T_2})\subseteq \nae(\mm|^S)$, and note that
$|R_1|\geq |T_1|-1$, $|R_2| \geq |T_2|-1$. 
If $T_2=\emptyset$ we simply use $R_1$. 
Likewise if $T_1=\emptyset$, we use $R_2$. 

If both $T_1$ and $T_2$ are nonempty, $\nae(\mm|^{T_2})\subseteq \nae(\mm|^{S})$, and $|\nae(\mm|^{T_2})|\geq |T_2|-1$. 
Now use the matching $f$. The set of rows $f(T_1)$ lies in $R'$ and is therefore disjoint from
$\nae(\mm|^{T_2})$. 
Moreover since $T_2\neq \emptyset$, every entry $(t,j)$ for $t\in T_2, j \in R'$ is white. On the other hand due to the construction of $f$, for every $t\in T_1$ the entry $(t,f(t))$ is non-white. Therefore every row in $f(T_1)$ is in
$\nae(\mm|^{T_1\cup T_2})$. So $|\nae(\mm|^{T_1\cup T_2})| \geq |T_2|-1+|T_1|$. 
Thus $\oep(\mm')\geq -1$. \end{proof}

\section{From NAE to Rank: Proof of Theorem~\ref{naethm}~\ref{partb}} \label{sec:naepf}
Recall we are to show:~\ref{naethm}~\ref{partb}: \textit{$\hset(\mm)$ has full column rank if 
$\oep(\mm)\geq -1$.}

\begin{proof} The case $k=1$ is trivial.
Now suppose $k\geq 2$ and that Theorem~\ref{naethm}~\ref{partb} holds for all $k'<k$. 
Any constant rows of $\mm$ affect neither the hypothesis nor the conclusion, so remove them, leaving $\mm$ with at least $k-1$ rows. Now pick any set, $C$, of $k-1$ columns of $\mm$. By Theorem~\ref{naethm}~\ref{parta}
there are some $k-2$ rows of $\mm$, call them $R'$, on which 
$\oep(\mm|_{R'}^{C})=-1$. Let $v$ be a row of $\mm$ outside $R'$. Call the rows of $\mm$ apart from $v$, $R''$. 
Since $R''$ contains $R'$, 
by induction $\dim \rowspace \hset(\mm|_{R''}^{C}) = k-1$.
Therefore $U:=\rowspace \hset(\mm|_{R''}) \subseteq \rset^k$ is of dimension at least $k-1$. We claim now that $\dim U = k$.
 
Suppose to the contrary that $\dim U=k-1$. 
If $v_\odot(U) \subseteq U$ then as proven earlier in Theorem~\ref{invU}, $U$ respects $B(v)$. Since $v$ is nonconstant, $B(v)$ is a partition of $[k]$ into $\ell\geq 2$ nonempty blocks $B(v)_{(i)}$, and $U=\bigoplus_{i=1}^\ell U_{(i)}$ with $U_{(i)}=P_{(i)} U_{(i)}$. So there is some $i_0$ for which $U_{(i_0)}\subset V_{(i_0)}$; specifically, 
$U_{(i)} = V_{(i)}$ for all $i \neq i_0$, and $\dim U_{(i_0)}= \dim V_{(i_0)}-1$. 
 Since $|B(v)_{(i_0)}|<k$, we know by induction that the rows of $\hset(\mm)$ span $V_{(i_0)}$. Thus in fact $U=\rset^k$.
 (Further detail for the last step: let $w\in\rset^k$. Since the rows of $\hset(\mm)$ span $V_{(i_0)}$, there is a $w'\in \rowspace \hset(\mm)$ s.t.\ $P_{(i_0)}w'=P_{(i_0)}w$. Moreover since $U_{(i)} = V_{(i)}$ for all $i \neq i_0$, there is a $w'' \in U$ s.t.\ $w''=(I-P_{(i_0)})(w-w')$. Then $w'+w'' \in \rowspace \hset(\mm)$, and $w'+w''=w$.) \end{proof}

\section{Motivation} \label{sec:motivation}
Consider \textit{observable} random variables 
$X_1,\ldots,X_n$ that are statistically independent conditional on $H$, a 
 \textit{hidden} random variable $H$ supported on $\{1,\ldots,k\}$. (See causal diagram.)

\[ \entrymodifiers={++[o][F-]} \xymatrix{
*{} & *{} & \dc{H} \ar@{..>}[lld] \ar@{..>}[ld] \ar@{..>}[d] \ar@{..>}[rd] \ar@{..>}[rrd] \ar@{..>}[rrrd] &*{} & *{} & *{}  \\
X_1  & X_2  & \cdots & X_i  & \cdots  & X_{n}
} \label{Fig:MixProd} \]

The most fundamental case is that the $X_i$ are binary. Then
we denote $\mm_{i}^j= \Pr(X_i=1 | H=j)$. 
The model parameters are $\mm$ along with a probability distribution (the \textit{mixture} distribution) $\pi=(\pi_1,\ldots,\pi_k)$ on $H$.

Finite mixture models were
pioneered in the late 1800s in~\cite{Newcomb86,Pearson94}. The problem of learning such distributions has drawn a great deal of attention. For surveys see, e.g.,~\cite{Everitt1981,TSM85,lindsay1995mixture,McLachlanLR19}. For some algorithmic papers on discrete $X_i$, see~\cite{KMRRSS94,CGG01,FM99,CR08,FOS08,AHK12,RSS14,LRSS15,CM19,GMRS20b}. The source identification problem is that of computing $(\mm,\pi)$ from the joint statistics of the $X_i$. Put another way, the problem is to invert the multilinear moment map
\begin{align*} \mu:(\mm,\pi) & \to \rset^{2^{[n]}}\\ 
\mu(\mm,\pi)_S&= \Pr(X_S=1) \quad \text{ where  } S\subseteq [n], \; X_S=\prod\nolimits_{i\in S}X_i \\
&= \mm_S \cdot \pi^\top
\end{align*}
The last line shows the significance of $\hset(\mm)$ to mixture model identification, since $\mm_{S}^j=\Pr(X_S=1|H=j)$.

\paragraph{Connection to $\rank \hset(\mm)$.} 
In general $\mu$ is not injective (even allowing for permutation among the values of $H$). For instance it is clearly not injective if $\mm$ has two identical columns (unless $\pi$ places no weight on those). More generally, and assuming all $\pi_j>0$, it cannot be injective unless $\hset(\mm)$ has full column rank.

One sufficient condition for injectivity, due to~\cite{TahmasebiMM18}, is that there be $2k-1$ ``separated'' observables $X_i$; $X_i$ is separated if all $\mm_{i}^j$ are distinct, or in our terminology, if no color recurs in $\mm_i$. 
(Further~\cite{GMRS20b}, one can lower bound the distance between $\mu(\mm,\pi)$ and any $\mu(\mm',\pi')$ in terms of $\min_{i,j} |\mm_{i}^j-\mm_i^{j'}|$ and
the distance between $(\mm,\pi)$ and $(\mm',\pi')$.)

A weaker sufficient condition for injectivity of $\mu$, due to~\cite{GMRS20b},
 is that for every $i\in[n]$ there exist two disjoint sets $A,B\subseteq [n]-\{i\}$ such that 
$\hset(\mm|_A)$ and $\hset(\mm|_B)$ have full column rank.
(It is not known whether two disjoint such $A,B$ are strictly necessary, but the implied $n\leq 2k-1$ is in general best possible~\cite{RSS14}.)
\newpage
\bibliographystyle{plainurl}
\bibliography{refs}
\end{document}